\pdfoutput=1
\documentclass[sigconf]{acmart}
\setcopyright{none}
\renewcommand\footnotetextcopyrightpermission[1]{} 
\makeatletter
\renewcommand\@formatdoi[1]{\ignorespaces}
\makeatother

\usepackage{algorithm2e}
\usepackage{multirow}
\usepackage{rotating}
\usepackage{lineno}
\usepackage{caption}
\usepackage{subcaption}
\usepackage{graphicx}
\usepackage{wrapfig}
\usepackage{amsmath,amsfonts}
\usepackage{amsthm}
\usepackage{wrapfig}
\theoremstyle{definition}

\theoremstyle{plain}
\newtheorem{assumption}{Assumption}
\newtheorem{remark}{Remark}
\newtheorem{corollary}{Corollary}
\newtheorem{conjecture}{Conjecture}
\newtheorem{theorem}{Theorem}
\newtheorem{lemma}{Lemma}
\usepackage{longtable}
\usepackage{caption}
\usepackage{subcaption}
\usepackage{physics}

\usepackage{array}
\usepackage{siunitx}
\usepackage[normalem]{ulem}
\usepackage{colortbl}
\usepackage{hhline}
\usepackage{calc}
\usepackage{tabularx}
\usepackage{threeparttable}
\usepackage{adjustbox}
\usepackage{soul}
\usepackage{hyperref}

\begin{document}

\title{GCF: Generalized Causal Forest for Heterogeneous Treatment Effects Estimation in Online Marketplace}
\thanks{Accepted at \textit{KDD-22 Workshop on Decision Intelligence and Analytics for Online Marketplaces: Jobs, Ridesharing, Retail, and Beyond}}

\author{Shu Wan}
\email{swan@asu.com}
\authornote{Both authors contributed equally to this research.}
\authornote{This work has been done when the author was with Didi Chuxing.}
\author{Chen Zheng}
\email{zhengchen04@meituan.com}
\authornotemark[1]
\authornotemark[2]
\affiliation{\institution{Didi Chuxing}  \country{China}}

\author{Zhonggen Sun}
\email{sun@didiglobal.com}
\author{Mengfan Xu}
\authornotemark[2]
\email{MengfanXu2023@u.northwestern.edu}
\affiliation{\institution{Didi Chuxing}  \country{China}}

\author{Xiaoqing Yang}
\email{yangxiaoqing@didiglobal.com}
\affiliation{\institution{Didi Chuxing}  \country{China}}

\author{Hongtu Zhu}
\authornotemark[2]
\email{htzhu@email.unc.edu}
\affiliation{\institution{University of North Carolina at Chapel Hill} \country{United States}}

\author{Jiecheng Guo}
\email{jasonguo@didiglobal.com}
\affiliation{\institution{Didi Chuxing} \country{China}}


\ccsdesc[500]{Computing methodologies~Classification and regression trees}
\ccsdesc[300]{Computing methodologies~Kernel methods}
\ccsdesc[300]{Computing methodologies~Modeling methodologies}
\ccsdesc[300]{Applied computing~Marketing}
\ccsdesc[500]{Mathematics of computing~Nonparametric statistics}

\keywords{Treatment effects estimation, continuous treatment, uplift modeling, online marketplace}

\begin{abstract}
Uplift modeling is a rapidly growing approach that utilizes causal inference and machine learning methods to directly estimate the heterogeneous treatment effects, which has been widely applied to various online marketplaces to assist large-scale decision-making in recent years. The existing popular models, like causal forest (CF), are limited to either discrete treatments or posing parametric assumptions on the outcome-treatment relationship that may suffer model misspecification. However, continuous treatments (e.g., price, duration) often arise in marketplaces. To alleviate these restrictions, we use a kernel-based doubly robust estimator to recover the non-parametric dose-response functions that can flexibly model continuous treatment effects. Moreover, we propose a generic distance-based splitting criterion to capture the heterogeneity for the continuous treatments. We call the proposed algorithm generalized causal forest (GCF) as it generalizes the use case of CF to a much broader setting. We show the effectiveness of GCF by deriving the asymptotic property of the estimator and comparing it to popular uplift modeling methods on both synthetic and real-world datasets. We implement GCF on Spark and successfully deploy it into a large-scale online pricing system at a leading ride-sharing company. Online A/B testing results further validate the superiority of GCF.
\end{abstract}

\maketitle

\pagestyle{plain} 

\section{Introduction}
\label{sec:introduction}
The rising of ride-sharing platforms such as DiDi, Uber, and Lyft, helps to provide convenient mobility services for riders and flexible job opportunities for drivers. However, it is extremely challenging for ride-sharing platforms to efficiently balance the demand and supply given the highly dynamic nature in this two-sided marketplace. For example, in a short period of time, the number of idle drivers in a given area can be seen as a constant since vehicle re-positioning takes time. On the other hand, the riders' requests can easily shift due to various reasons such as the change of price, disturbance of ETA and severity of road congestion. Therefore, adjusting the demand is at the heart of the ride-sharing platforms strategies and often draws more attention \citep{lam2017demand, shapiro2018density}.

\vspace{-4mm}
\begin{figure}[h]
  \centering
  \includegraphics[width=0.25\linewidth]{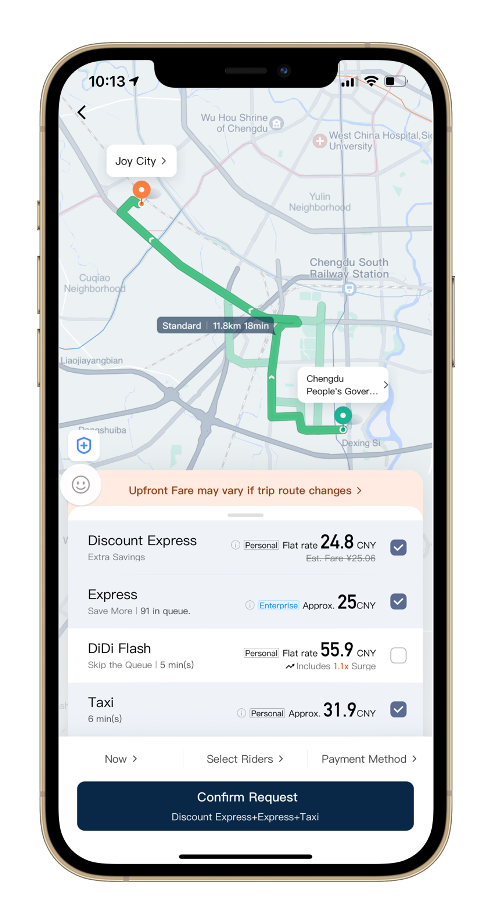}
  \vspace*{-4mm}
  \caption{A snapshot of Trip Request}
  \label{snapshot}
\end{figure}
\vspace{-4mm}
A typical journey of how a rider requests a trip on a ride-sharing app is shown in Figure \ref{snapshot}. A rider first decides the pick-up and the drop-off locations and sees the estimated price of multiple mobility options, then responds by either clicking on the 'Confirm Request' button or ending up with no request. Among all information shown on this page, we can observe that price plays a pivotal role in the riders' choices. In our case, pricing for trips is at the origin-destination-time (ODT) level which guarantees a consistent experience for all riders on the same ODT. As stated earlier, the number of drivers is relatively unchanged in a given ODT,  which leaves opportunities to use discount strategies to stimulate riders' requests when the supply is excessive as requests usually increase when the price goes down. But this is not a trivial decision to make as a too low price can result in excessive requests and subsequently longer waiting time, henceforth hurting riders' experience and deteriorating the efficiency of the marketplace. On the flip side, if the incentive is not strong enough, it may not be sufficient to stimulate enough requests to balance idle drivers on the same ODT. The optimal discount can only be achieved when the demand price curve is accurately estimated. However, the curve may significantly differ across different ODTs.

In Figure~\ref{odt}, for example, we present how demand varies with price on different ODTs. Therefore, the same discount for different ODTs makes little sense. In other words, the platforms should assign appropriate discounts to ODT accordingly by leveraging ODT's specific information and real-time supply-demand relationship to identify the impact of discounts on the demand curve.

\begin{figure}[h]
  \centering
  \includegraphics[width=1\linewidth]{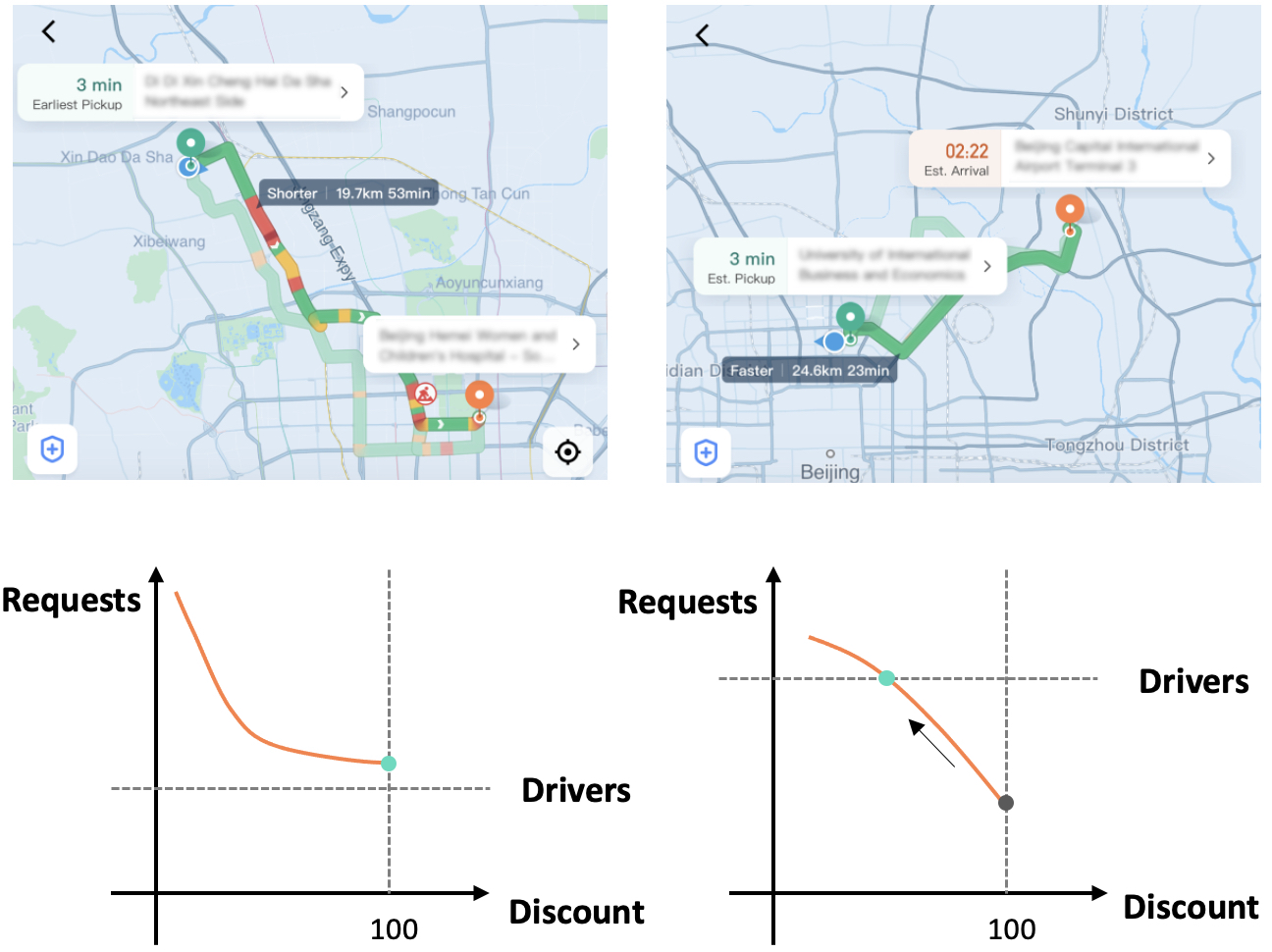}
    \vspace*{-4mm}
  \caption{This plot displays different discount strategies at different ODT scenarios. When the demand is over supply, no discount should be charged (left plot), while in the off-peak time, we use discount to stimulate the requests to balance the marketplace.}
  \label{odt}
\end{figure}
More generally, the question is how to estimate the discount effect on demand under different scenarios, formally described as the problem of heterogeneous treatment effect (HTE) estimation in the field of causal inference, which has been of growing interests for decision-makers in a wide spectrum of contexts. It uncovers the effect of interventions at sub-group levels, thereby providing highly tailored suggestions rather than a one-size-fits-all policy. Moreover, for online ride-sharing marketplaces, (multiple) continuous treatments are prevalent as multiple travel options are available as shown in Figure~\ref{snapshot}. Estimating the causal effect under continuous treatments presents a challenge for the marketplace while maintaining crucial for maximizing its efficiency and performance.

A series of algorithms have been developed to address the problem of HTE estimation. The earliest solution dates back to when uplift modeling had the most appeals as in ~\cite{radcliffe2007using} and has been recently applied to online marketplaces such as ~\cite{zhao2017uplift, 10.1145/3447548.3467083}. However, these implementations fail to discuss how to mitigate confounding bias which is prevalent in observational data. In contrast, statistical and econometric methods, such as Causal Forest (CF)~\cite{athey2019generalized, chernozhukov2018double} directly consider the relationship between outcome and treatment in presence of confounding variables. Nevertheless, the theoretical property of the estimators is built on the assumption that the outcome is partially linear in the treatments. In practice, the effect of discounts on requests can be any function of treatments, as illustrated in Figure~\ref{odt}.
To address this problem, \citep{blundell2012measuring, kennedy2017nonparametric, colangelo2020double, singh2020reproducing} propose using nonparametric regression to solve nonlinear HTE estimations. Our work builds on the theoretical results in these works. Meanwhile, the scalability of an algorithm is key to its deployment to the online marketplace with massive data. In recent years, neural network based methods, such as \citep{shalit2017estimating, nie2021vcnet} are also been developed, but they lack interpretability which is important in the high stake settings like pricing strategy.

In this paper we overcome the aforementioned challenges by proposing generalized casual forest (GCF), a method that provides nonparametric HTE estimations for continuous treatments. GCF has shown its advantages over existing baselines on both synthetic datasets and real-world datasets, and demonstrated its high performance on online deployment at a leading ride-sharing company. Moreover, we implement GCF on Spark and obtain much higher computational efficiency by distributed computing, which pave the way to a wide application to massive online marketplaces.

The rest of the paper is organized as follows. Section $2$ introduces preliminary notations and backgrounds. Then in Section $3$, we formally propose GCF. We validate the performance of GCF by applying it to both synthetic and real-world datasets in Section $4$. Finally, in Section $5$, the practical effectiveness of GCF is demonstrated by its superior performance in an online experiment. The Spark implementation of GCF is also briefly introduced in this section. We conclude the paper with some discussions in Section $6$.

\section{Preliminaries}
\label{sec:3}

\subsection{Notations and Assumptions}

We first introduce the notations for HTE estimation with continuous treatments. Following the potential outcome framework in~\citep{neyman1923applications, rubin1974estimating}, we let $\boldsymbol{T}$ be the $d_t$-dim continuous treatment, $\boldsymbol{X} = (X^j)_{j=1}^{p_X}$ be the $p_X$-dim confounding variables, $\boldsymbol{U}$ be the $p_U$-dim outcome-specific covariates, $\boldsymbol{Z}$ be the $p_Z$-dim treatment-specific covariates independent of $\boldsymbol{U}$, and $Y$ be the outcome of interest. The population $\Omega: (\boldsymbol{X}, \boldsymbol{U}, \boldsymbol{Z}, Y, \boldsymbol{T}) \in \mathbb{R}^{p_X + p_U + p_Z + 1 + d_t}$ satisfies 
\begin{center}
    $Y = g(T,\boldsymbol{X},\boldsymbol{U}) + \epsilon ; T = f(\boldsymbol{X},\boldsymbol{Z}) + \nu$
\end{center}
where $\epsilon,\nu$ are noises of zero mean and $g: \mathbb{R}^{p} \times \mathbb{R} \rightarrow \mathbb{R}$ and $f: \mathbb{R}^{p}  \rightarrow \mathbb{R}$. The i.i.d. samples $\{(\boldsymbol{X}_i, \boldsymbol{U}_i, \boldsymbol{Z}_i, Y_i, T_{i}), i=1, \ldots, n\}$ are drawn from $\Omega$. In practice, the decomposition of $\boldsymbol{X}, \boldsymbol{U}, \boldsymbol{Z}$ from observed covariates is difficult or infeasible. Without additional specifications, we use $\boldsymbol{X}$ as a summary of observed covariates throughout the paper. The potential outcome under treatment $t$ is $Y_{(\boldsymbol{t})}$. Recall that Propensity Score (PS)~\citep{rubin1974estimating} for a discrete treatment is $P(\boldsymbol{T}=\boldsymbol{t}|\boldsymbol{X})$ , i.e. the probability for a unit receiving treatment $\boldsymbol{t}$ given the covariates $\boldsymbol{X}$. In the context of continuous treatments, ~\citep{rubin1974estimating, hirano2004propensity} introduce generalized propensity score (GPS), denoted by probability density function $\pi(\boldsymbol{T}=\boldsymbol{t}|\boldsymbol{X})$.

Our estimand of interest is CATE $\theta(\boldsymbol{t},\boldsymbol{X})$, formally defined as
\begin{align*}
\theta(\boldsymbol{t},\boldsymbol{X}) & = E[Y_{(\boldsymbol{t})}|\boldsymbol{X}] - E[Y_{(\boldsymbol{0})}|\boldsymbol{X}] 
\end{align*}
To identify $\theta(\boldsymbol{t},\boldsymbol{X})$ , common assumptions are made as in \citep{holland1986statistics, kennedy2017nonparametric}.
\begin{assumption}{Consistency:}\label{assumption:1} $E[Y|\boldsymbol{T}=\boldsymbol{t}] = E[Y_{(\boldsymbol{t})}|\boldsymbol{T}=\boldsymbol{t}]$, i.e. the outcome of any sample solely depends on its treatment.
\end{assumption}
\begin{assumption}{Ignorability:}\label{assumption:2} The potential outcomes $Y_{(\boldsymbol{T})}$ is independent of treatment $\boldsymbol{T}$ given covariates $\boldsymbol{X}$.
\end{assumption}
\begin{assumption}{Positivity:}\label{assumption:3} The GPS $\pi(\boldsymbol{T}=\boldsymbol{t}|\boldsymbol{X}) > p_{min} >0$, $\forall \boldsymbol{t},\boldsymbol{X}$, i.e. the density is bounded away from 0. 
\end{assumption}
Under Assumption 1-3, we have 
\begin{align*}
\theta(\boldsymbol{t},\boldsymbol{X}) & = E[Y_{(\boldsymbol{t})}|\boldsymbol{X}] - E[Y_{(\boldsymbol{0})}|\boldsymbol{X}] = E[Y|\boldsymbol{T} = \boldsymbol{t}, \boldsymbol{X}] - E[Y|\boldsymbol{T} = \boldsymbol{0}, \boldsymbol{X}] \\
            & = E[g(\boldsymbol{t},\boldsymbol{X})|\boldsymbol{T} = \boldsymbol{t}, \boldsymbol{X}] - E[g(\boldsymbol{t},\boldsymbol{X})|\boldsymbol{T} = \boldsymbol{0}, \boldsymbol{X}]
\end{align*}
where the first equality holds by Assumption \ref{assumption:1} and the second one holds by Assumption \ref{assumption:2}. Positivity is indispensable for the conditional expectation to be well-defined in the last line but is often too strong. In practice, it can be reduced to weak positivity as follows.
\begin{assumption}{Weak Positivity:}\label{assumption:4} The variance $ \sigma(\pi) > \sigma_{min} >0 $ where $\sigma(\pi) = \int_{\boldsymbol{t}} \boldsymbol{t}^2 \cdot \pi(\boldsymbol{T} = \boldsymbol{t}|\boldsymbol{X})d\boldsymbol{t} - \left(\int_{\boldsymbol{t}} \boldsymbol{t} \cdot \pi(\boldsymbol{T} = \boldsymbol{t}|\boldsymbol{X})d\boldsymbol{t}\right)^2$.
\end{assumption}

\subsection{Dose-Response Function}
For continuous treatments, the treatment effect can be fully characterized by Dose-Response Function (DRF), formally defined as $\mu(\boldsymbol{t}) :=  E[Y|\boldsymbol{T} = \boldsymbol{t}]$.
Generally speaking, DRF can be any function of $\boldsymbol{t}$, either parametric or non-parametric and thereby motivating our proposed method. The estimand of interest related to DRF usually depends on the contexts~\citep{galagate2016causal}. Examples include treatment effect $\left(E[Y_{(\boldsymbol{t})}] - E[Y_{(\boldsymbol{0})}]\right)$, partial effect $\left(\dfrac{E[Y_{(t_2)}] - E[Y_{(t_1)}]}{t_2-t_1}\right)$ and elasticity $\left(\dfrac{\partial log(E[Y_{(t)}])}{\partial log(t)}\right)$. Specifically, conditional DRF (CDRF) $\mu(\boldsymbol{t},\boldsymbol{X}) = E[Y|\boldsymbol{T} = \boldsymbol{t}, \boldsymbol{X}]$
characterizes the outcome with interventions at individual or subgroup levels. It is a proxy for DRF on samples in the subgroup and leads to $\theta(\boldsymbol{t},\boldsymbol{X})$ by noting that $
\theta(\boldsymbol{t},\boldsymbol{X}) = \mu(\boldsymbol{t},\boldsymbol{X})-\mu(\boldsymbol{0},\boldsymbol{X})$, i.e. $\mu(\boldsymbol{t},\boldsymbol{X}) $.

\subsection{Kernel Regression and Double/Debiased Estimators}
For non-parametric function approximations, kernel regression ~\citep{fan2018local} works with theoretical guarantees. Precisely, to model the non-parametric relationship $y= g(\boldsymbol{x})+\epsilon$ given data $(\boldsymbol{X}_i,Y_i)$, kernel regression produces an estimator by $\hat{g}(\boldsymbol{x}) =  \dfrac{\sum_{i=1}^{n} K_h(\boldsymbol{x},\boldsymbol{X}_i) \cdot y_i}{\sum_{i=1}^{n} K_h(\boldsymbol{x},\boldsymbol{X}_i)}$
where $K_h(\boldsymbol{x},\boldsymbol{X}_i)$ is a scaled kernel density with $h$ being the bandwidth. Typical choices of kernel include Uniform, Epanechnikov, Biweight, Triweight and Gaussian.

An asymptotically unbiased estimator combines kernel regression with doubly robust estimators in the context of DRF. Formally, the estimator for DRF $\mu(\boldsymbol{t})$ is proposed as 
$$\hat{\mu}(\boldsymbol{t}) = \dfrac{1}{n} \cdot \sum_{i=1}^{n}\left(\hat{\mu}(\boldsymbol{t},\boldsymbol{X}_i)+ \frac{K_h(\boldsymbol{T}_i,\boldsymbol{t})}{\hat{\pi}(\boldsymbol{t}|\boldsymbol{X}_i)} \cdot \left(Y_i - \hat{\mu}(\boldsymbol{t},\boldsymbol{X}_i)\right)\right)$$ by using two-stage estimating technique and kernel-based DML method as proposed in \citep{colangelo2020double}. 
First, we estimate the CDRF $\mu(\boldsymbol{t},\boldsymbol{X})$ with $\hat{\mu}(\boldsymbol{t},\boldsymbol{X})$ and the GPS $\pi(\boldsymbol{T} = \boldsymbol{t}|\boldsymbol{X})$ with $\hat{\pi}(\boldsymbol{t}|\boldsymbol{X})$ by means of any machine learning methods. Plugging the nuisance estimators then gives us the kernel-based double/debiased estimator.

\section{Generalized Causal Forest}
\label{sec:gcf}
In this section, we formally present the proposed algorithm, namely GCF. It relaxes the partial linear assumption on treatment response relationship in CF by considering a new splitting criterion with non-parametric DRF and estimating it with the kernel-based doubly robust estimator. In what follows, we show a workflow of GCF at both the training stage and the prediction stage, followed by elaborating on the splitting criterion, CATE estimators, and its asymptotic property. The details on GCF's practical tweaks and Spark implementation are given in the supplementary.

\RestyleAlgo{ruled}

\begin{algorithm}[h]
\caption{Generalized Causal Forest}\label{alg:Framwork} 

\textbf{Input} {dataset $\Omega = (\boldsymbol{X}_i,T_i,Y_i), i=1,\ldots,n$; number of trees $B$; number of features sampled for growing a tree $mtry$; the minimum sample size on each leaf node $min.node.size$; honesty fraction $\alpha$; tolerance $\tau$; positivity threshold $\zeta$}\;
\textbf{Trainings}
\Begin{
Split $\Omega$ with a ($\alpha$,$1-\alpha$) ratio to get $\Omega_1$ and $\Omega_2$ for honesty\;
Pre-train outcome regression model $\hat{\mu}$ and treatment density estimation $\hat{\pi}$ on sample $\Omega_1$\;
$b \longleftarrow 1$\;
\While{$b \leq B$}{
      Sample a feature set $\boldsymbol{X}^S$ of size $mtry$ from $\boldsymbol{X}$\;
      \For{The stopping rule is not satisfied}{
      Identify each parent node $P$\;
      Compute the splitting criterion $\Delta(\cdot)$ (\ref{splitting creterion}) over the samples in $\Omega_1$ with $\boldsymbol{X}^S$\;
      Grow the tree $\mathcal{T}_b$ by splitting at the parent node $P$ according to the  $\Delta(\cdot)$\;
      }
      Assign samples in $\Omega_2$ to leaf nodes based on $\mathcal{T}_b$ and let $b \longleftarrow b + 1$\;
}
}
\textbf{Output} Causal forest with $B$ trees by recursive partitioning on $X$; node assignments for samples in $\Omega_2$\;
\textbf{Predictions}
\Begin{
CDRF estimation $\hat{\mu}_b(t,\boldsymbol{x})$ by the local weighted average over the outcomes of samples in $\Omega_2$ that falls into $\mathcal{L}_b(\boldsymbol{x})$\, as $\sum_{i=1}^{n}\dfrac{1_{X_i \in \mathcal{L}_b(\boldsymbol{x})} \cdot 1_{T_i=t}\cdot Y_i}{|\mathcal{L}_b(\boldsymbol{x})|}$;
 CATE estimation $\hat{\theta}_b(t,\boldsymbol{x}) = \hat{\mu}_b(t,\boldsymbol{x})-\hat{\mu}_b(0,\boldsymbol{x})$ and $\hat{\theta}(t,\boldsymbol{x}) = \dfrac{1}{B}\sum_{b=1}^{B}\hat{\theta}_b(t,\boldsymbol{x})$;
}
\end{algorithm}

GCF grows trees $\mathcal{T}_b$ by repeating the tree-growing process $B$ times through bootstrapping. At the training stage, we construct a tree $\mathcal{T}_b$ by recursive partition based on maximizing a newly proposed splitting criterion $\Delta(c_1,c_2)$. It is proportional to the discrepancy between the CDRF $\theta$ given by the left and right child nodes, respectively, thereby reflecting the underlying heterogeneity and being an estimator for CDRF with a theoretical guarantee.

Our algorithm is implemented on Spark for large-scale data processing and the mechanism of the tree-growing process is different from that in CF. Precisely, data is stored at the master machine and trees are cloned to each branch machine. Data is randomly distributed to branch machines for parallel computation and recollected to the master machine for integration. The tree will be updated by the integrated criterion on each branch machine. This distributed framework leverages the computational efficiency of multiple machines and speeds up the training process. 

Trees stop growing when some stopping criteria are met, such as when the sample size of child nodes is smaller than $min.node.size$ or the splitting criterion $\Delta(c_1,c_2) \leq min.info.gain$. We borrow the honesty principle in ~\citep{athey2019generalized} by partitioning the training samples into two parts where one is for growing a tree and the other only involves CATE estimations on the leaf nodes. Each training data can either be utilized to estimate CATE or contribute to growing a tree. At the prediction stage, the leaf node of given samples $\boldsymbol{x}$ on each tree $\mathcal{T}_b$ is denoted by $\mathcal{L}_b(\boldsymbol{x})$. The estimation of CDRF $\hat{\mu}_b(\boldsymbol{t},\boldsymbol{x})$ on $\mathcal{L}_b(\boldsymbol{x})$ takes a local weighted average on training samples in $\mathcal{L}_b(\boldsymbol{x})$, with weight $\alpha_i$ being
\begin{equation}\label{equa:1}
    \alpha_i = \sum_{i=1}^{n}\dfrac{1_{X_i \in \mathcal{L}_b(\boldsymbol{x})} \cdot 1_{\boldsymbol{T}_i=\boldsymbol{t}}\cdot Y_i}{|\mathcal{L}_b(\boldsymbol{x})|}
\end{equation}
Accordingly, the CATE estimator on each tree $b$ is given by $\hat{\theta}_b(\boldsymbol{t},\boldsymbol{x}) = \hat{\mu}_b(\boldsymbol{t},\boldsymbol{x})-\hat{\mu}_b(\boldsymbol{0},\boldsymbol{x})$. The final CATE estimation is the average of $\mathcal{L}_b(\boldsymbol{x})$ over $B$ trees, as $\hat{\theta}(\boldsymbol{t},\boldsymbol{x}) = \dfrac{1}{B}\sum_{b=1}^{B}\hat{\theta}_b(\boldsymbol{t},\boldsymbol{x})$.
\subsection{Splitting Criterion}
\label{sub:splitting}

Given a parent node $P$ and training samples with covariates $\boldsymbol{X}_\omega$, our splitting criterion $\Delta(C_1,C_2)$ for left child node $C_1$ and right child node $C_2$ is proposed as 
\begin{equation}
   \Delta = \Delta(C_1,C_2)= \dfrac{n_{C_1}n_{C_2}}{n_P} \cdot ||\hat{\theta}_{C_1} - \hat{\theta}_{C_2}||_F^{2} = \dfrac{n_{C_1}n_{C_2}}{n_P} \cdot D(\hat{\theta}_{C_1}, \hat{\theta}_{C_2})
   \label{splitting creterion}
\end{equation}
where $\hat{\theta}_{C_1}$, $\hat{\theta}_{C_2}$ are the CATE estimators in the Banach space of $\theta_x$ on $C_1$ and $C_2$, respectively. The sample sizes of parent node, left child node and right child node are $n_P$, $n_{C_1}$ and $n_{C_2}$, respectively. The ratio $\dfrac{n_{C_1} n_{C_2}}{n_P^2}$ in $\Delta(C_1,C_2)$ is to balance the sample sizes of two child nodes. Distance metric $D$ measures the distance between $\hat{\theta}_{C_1}$ and $\hat{\theta}_{C_1}$ in the Banach space $ \{h:\mathcal{R}^{d_t} \to \mathcal{R}\}\}$ and thereby representing the heterogeneity. Some commonly used metrics \citep{dette2018equivalence} are $
D_{1} = \int_t|\hat{\theta}_{C_1}(t) - \hat{\theta}_{C_2}(t)|dt,
D_{2} = \int_t|\hat{\theta}_{C_1}(t) - \hat{\theta}_{C_2}(t)|^2dt,
D_{\infty} = \max_t|\hat{\theta}_{C_1}(t) - \hat{\theta}_{C_2}(t)|$ induced by $L_{1}$, $L_{2}$ and $L_{\infty}$ norms, respectively.
We next show how to estimate CDRF $\hat{\theta}_{C_1}$ and $\hat{\theta}_{C_1}$ using kernel-based doubly robust estimators aforementioned. 
\subsection{CATE Estimation} \label{sub:non-para}

We point out the kernel-based DML estimator in \citep{colangelo2020double} can be adapted to the DRF estimation in GCF and guide tree splittings recursively.

Specifically, at each splitting, herein we propose to estimate the DRF of child node $C$ by
\begin{equation*}
    \Tilde{\mu}_C(t) =  \frac{1}{|n_C|}\sum_{i \in C}\left(\hat{\mu}(t,\boldsymbol{X}_i)+ \frac{K_h(T_i-t)}{\hat{\pi}(t|\boldsymbol{X}_i)} \cdot (Y_i -\hat{\mu}(t,\boldsymbol{X}_i))\right) 
\end{equation*}
where $\hat{\mu}(t,\boldsymbol{X}_i)$ and $\hat{\pi}(t|\boldsymbol{X}_i)$ are the pretrained estimator for $\mu(t,\boldsymbol{X})$ and GPS at $\boldsymbol{X}_i$, respectively. It can be viewed as integrating estimators for CDRF, DRF, ATE at child-node level repeatedly, instead of a one-step estimation as a whole in~\citep{colangelo2020double}
\paragraph{CDRF estimator:}$\Tilde{\mu}(t, \boldsymbol{X}_i) = \hat{\mu}(t,\boldsymbol{X}_i) + \frac{K_h(T_i-t)}{\hat{\pi}(t|\boldsymbol{X}_i)} \cdot \left(Y_i - \hat{\mu}(t,\boldsymbol{X}_i)\right)$
\paragraph{DRF and ATE estimator}
The ATE estimator is given by
\begin{align*}
& \Tilde{\mu}(t) = \dfrac{\sum_{i \in C_1} \Tilde{\mu}(t, \boldsymbol{X}_i)}{|n_{C_1}|}, \quad \hat{\theta}_{C_1}(t) = \Tilde{\mu}(t) - \Tilde{\mu}(0)
\end{align*}

In estimating CATE, the smoothness condition of DRF plays an important role. For convex and smooth DRF with $L_p$ norm, we can approximate $\hat{\theta}_{C_1}(t)$ with a closed-form $t \cdot \frac{\partial \tilde{\mu}(t)}{\partial t}$ where $\frac{\partial \tilde{\mu}(t)}{\partial t}$, namely PDRF, is denoted by $\hat{\Phi}(t)$. To see why this holds, by the convexity and specifying $||\cdot||_F$ to be $L_p$ norm, we have that 
\begin{align*}
   & \dfrac{n_{C_1}n_{C_2}}{n_P} ||\hat{\Phi}_{C_1} - \hat{\Phi}_{C_2}||_F^{2} (T_{min})^2 \leq \Delta \leq \dfrac{n_{C_1}n_{C_2}}{n_P} ||\hat{\Phi}_{C_1} - \hat{\Phi}_{C_2}||_F^{2}  (T_{max})^2
\end{align*}
i.e.value $\Delta$ is equivalent to $\tilde{\Delta} = \dfrac{n_{C_1}n_{C_2}}{n_P} \cdot ||\hat{\Phi}_{C_1} - \hat{\Phi}_{C_2}||_F^{2}$, the objective at the tree splitting in Algorithm~\ref{alg:Framwork}.  

For Gaussian kernels and 1-dim treatments, $\Phi(t)$ is explicitly as $ \frac{\sum_{i=1}^{n}K'_h(T_i-t)Y_i}{\sum_{i=1}^{n}K_h(T_i-t)}-  \big[\sum_{i=1}^{n}K_h(T_i-t)Y_i\big]\frac{\sum_{i=1}^{n}K'_h(T_i-t)}{[\sum_{i=1}^{n}K_h(T_i-t)]^2}$.
However, for general convex DRF, the PDRF may be discontinuous or even ill-defined. For example, non-differentiable kernel functions like uniform, epanechnikov, biweight and triweight, $\Phi$ can only be estimated via numerical approximations
$\dfrac{E[\mu(t+\delta)-\mu(t)]}{\delta}$. 

\subsection{Asymptotic Property}

We state the convergence property of estimators under mild assumptions. We show that the final estimator for CATE $\theta(t,x)$, or equivalently $\theta_x(t)$, by GCF is doubly robust in the sense that it is asymptotically unbiased when the estimator for nuisance function $\pi(\mathbf{x})$ or $\mu(\mathbf{x})$ is unbiased. The result is summarized as Theorem~\ref{the:theorem1} and the proof for Theorem~\ref{the:theorem1} is in Supplementary~\ref{ssec:proof1}.

\begin{theorem}\label{the:theorem1}
Under certain assumptions, the final estimator $\hat{\theta}_x(t)$ given by the proposed GCF converges to CATE $\theta_x(t)$ in the functional space as the number of samples $n \to \infty$, i.e.
\begin{equation*}
  \norm{\hat{\theta}_x(t) - \theta_x(t)}_F \to 0 \text{ as }  n \to \infty
\end{equation*}
\end{theorem}

\begin{corollary}
The convergence can be of any form depending on the chosen norm of $\norm{\cdot}_F$. For example, the result is $L_p$ convergence when $F$ is $L_p$ with $1< p < \infty $ and will converge uniformly when $F$ is $L_{\infty}$. 
\end{corollary}

Our theoretical result distinguishes from the existing ones as follows. First, while the results of~\cite{DR2020} only hold for binary treatments, we allow for continuous treatments. 
While~\cite{athey2019generalized} propose the convergence of the point estimator for the parameter of CATE with continuous treatments, their DRF is presumed to be linear. We remove the assumption and show its convergence in the functional space. Lastly, the non-linear doubly robust estimator given by~\cite{colangelo2020double} is asymptotically unbiased with respect to ATE, rather than CATE.

\section{Experiment}
\label{sec:experiment}
This section provides numerical evidence of GCF's performance on simulation and real-world datasets. GCF is implemented on Spark with distance metric $D = d_2$. We employ random forest (RF) \citep{breiman2001random}, CF, and Kennedy \citep{codeKennedy2017} as baseline methods, to show the effectiveness of GCF both on non-parametric treatment effects estimation and on computational efficiency. More specifically, RF takes less causality into account, CF has a similar tree-based framework for causal inference but may not reflect the non-linear relationship and Kennedy propose the non-parametric kernel estimator but with limited computational efficiency. 

\subsection{Evaluation}

For the evaluation on synthetic datasets with ground truths, we follow what is used in \citep{hill2011bayesian}. Specifically, we use PEHE and RMSE to evaluate the bias and variance of estimators by different methods. Let $n$ be the number of samples,  $\hat{\theta}_i^t$ and $\theta_i^t$ be the predicted and true treatment effect of $i^{th}$ sample for treatment $t$, and $\hat{\theta}^t = \sum_{i=1}^{t} \hat{y}_i^t - \hat{y}_i^0$. PEHE, RMSE, and ADRF are $\dfrac{\sum_{i=1}^{n}(|\hat{\theta}^t - \theta^t|))}{n}$, $\sqrt{\dfrac{\sum_{t=1}^{n}(\hat{\theta}^t - \theta^t)^2}{n}}$, $\hat{y(t)} = \int_{x} \hat{\mu}(t,x) dx$. For real-world problems without ground truth, Qini Curve and Qini score are utilized~\citep{gutierrez2017causal}. A larger Qini Score reflects a better model for HTE estimations.

  \providecommand{\huxb}[2]{\arrayrulecolor[RGB]{#1}\global\arrayrulewidth=#2pt}
  \providecommand{\huxvb}[2]{\color[RGB]{#1}\vrule width #2pt}
  \providecommand{\huxtpad}[1]{\rule{0pt}{#1}}
  \providecommand{\huxbpad}[1]{\rule[-#1]{0pt}{#1}}

\begin{table*}[htb]
\begin{centerbox}
\begin{threeparttable}
 \setlength{\tabcolsep}{0pt}
 \caption{Simulation results on datasets with different DRFs where number of samples $n = 1000, p_X=50, p_Y=5, p_Z=5$. Standard errors are in parenthesis over 100 simulations.}\label{tab: RMSE}
\begin{tabularx}{0.95\textwidth}{p{0.135714285714286\textwidth} p{0.135714285714286\textwidth} p{0.135714285714286\textwidth} p{0.135714285714286\textwidth} p{0.135714285714286\textwidth} p{0.135714285714286\textwidth} p{0.135714285714286\textwidth}}

\hhline{>{\huxb{0, 0, 0}{0.4}}->{\huxb{0, 0, 0}{0.4}}->{\huxb{0, 0, 0}{0.4}}->{\huxb{0, 0, 0}{0.4}}->{\huxb{0, 0, 0}{0.4}}->{\huxb{0, 0, 0}{0.4}}->{\huxb{0, 0, 0}{0.4}}-}
\arrayrulecolor{black}

\hhline{}
\arrayrulecolor{black}

\multicolumn{1}{!{\huxvb{0, 0, 0}{0}}p{0.135714285714286\textwidth}!{\huxvb{0, 0, 0}{0}}}{\hspace{6pt}\parbox[b]{0.135714285714286\textwidth-6pt-6pt}{\huxtpad{3pt + 1em}\centering \textit{\textbf{{\fontsize{10pt}{12pt}\selectfont }}}\huxbpad{3pt}}} &
\multicolumn{2}{p{0.271428571428571\textwidth+2\tabcolsep}!{\huxvb{0, 0, 0}{0}}}{\hspace{6pt}\parbox[b]{0.271428571428571\textwidth+2\tabcolsep-6pt-6pt}{\huxtpad{3pt + 1em}\centering \textit{\textbf{{\fontsize{10pt}{12pt}\selectfont Polynomial}}}\huxbpad{3pt}}} &
\multicolumn{2}{p{0.271428571428571\textwidth+2\tabcolsep}!{\huxvb{0, 0, 0}{0}}}{\hspace{6pt}\parbox[b]{0.271428571428571\textwidth+2\tabcolsep-6pt-6pt}{\huxtpad{3pt + 1em}\centering \textit{\textbf{{\fontsize{10pt}{12pt}\selectfont Sinusoidal}}}\huxbpad{3pt}}} &
\multicolumn{2}{p{0.271428571428571\textwidth+2\tabcolsep}!{\huxvb{0, 0, 0}{0}}}{\hspace{6pt}\parbox[b]{0.271428571428571\textwidth+2\tabcolsep-6pt-6pt}{\huxtpad{3pt + 1em}\centering \textit{\textbf{{\fontsize{10pt}{12pt}\selectfont Exponential}}}\huxbpad{3pt}}} \tabularnewline[-0.5pt]

\hhline{>{\huxb{255, 255, 255}{0.4}}->{\huxb{0, 0, 0}{0.4}}->{\huxb{0, 0, 0}{0.4}}->{\huxb{0, 0, 0}{0.4}}->{\huxb{0, 0, 0}{0.4}}->{\huxb{0, 0, 0}{0.4}}->{\huxb{0, 0, 0}{0.4}}-}
\arrayrulecolor{black}

\multicolumn{1}{!{\huxvb{0, 0, 0}{0}}p{0.135714285714286\textwidth}!{\huxvb{0, 0, 0}{0}}}{\hspace{6pt}\parbox[b]{0.135714285714286\textwidth-6pt-6pt}{\huxtpad{3pt + 1em}\raggedright \textbf{{\fontsize{10pt}{12pt}\selectfont Methods}}\huxbpad{3pt}}} &
\multicolumn{1}{p{0.135714285714286\textwidth}!{\huxvb{0, 0, 0}{0}}}{\hspace{6pt}\parbox[b]{0.135714285714286\textwidth-6pt-6pt}{\huxtpad{3pt + 1em}\raggedright \textbf{{\fontsize{10pt}{12pt}\selectfont PEHE}}\huxbpad{3pt}}} &
\multicolumn{1}{p{0.135714285714286\textwidth}!{\huxvb{0, 0, 0}{0}}}{\hspace{6pt}\parbox[b]{0.135714285714286\textwidth-6pt-6pt}{\huxtpad{3pt + 1em}\raggedright \textbf{{\fontsize{10pt}{12pt}\selectfont RMSE}}\huxbpad{3pt}}} &
\multicolumn{1}{p{0.135714285714286\textwidth}!{\huxvb{0, 0, 0}{0}}}{\hspace{6pt}\parbox[b]{0.135714285714286\textwidth-6pt-6pt}{\huxtpad{3pt + 1em}\raggedright \textbf{{\fontsize{10pt}{12pt}\selectfont PEHE}}\huxbpad{3pt}}} &
\multicolumn{1}{p{0.135714285714286\textwidth}!{\huxvb{0, 0, 0}{0}}}{\hspace{6pt}\parbox[b]{0.135714285714286\textwidth-6pt-6pt}{\huxtpad{3pt + 1em}\raggedright \textbf{{\fontsize{10pt}{12pt}\selectfont RMSE}}\huxbpad{3pt}}} &
\multicolumn{1}{p{0.135714285714286\textwidth}!{\huxvb{0, 0, 0}{0}}}{\hspace{6pt}\parbox[b]{0.135714285714286\textwidth-6pt-6pt}{\huxtpad{3pt + 1em}\raggedright \textbf{{\fontsize{10pt}{12pt}\selectfont PEHE}}\huxbpad{3pt}}} &
\multicolumn{1}{p{0.135714285714286\textwidth}!{\huxvb{0, 0, 0}{0}}}{\hspace{6pt}\parbox[b]{0.135714285714286\textwidth-6pt-6pt}{\huxtpad{3pt + 1em}\raggedright \textbf{{\fontsize{10pt}{12pt}\selectfont RMSE}}\huxbpad{3pt}}} \tabularnewline[-0.5pt]

\hhline{>{\huxb{0, 0, 0}{0.4}}->{\huxb{0, 0, 0}{0.4}}->{\huxb{0, 0, 0}{0.4}}->{\huxb{0, 0, 0}{0.4}}->{\huxb{0, 0, 0}{0.4}}->{\huxb{0, 0, 0}{0.4}}->{\huxb{0, 0, 0}{0.4}}-}
\arrayrulecolor{black}

\multicolumn{1}{!{\huxvb{0, 0, 0}{0}}p{0.135714285714286\textwidth}!{\huxvb{0, 0, 0}{0}}}{\hspace{6pt}\parbox[b]{0.135714285714286\textwidth-6pt-6pt}{\huxtpad{3pt + 1em}\raggedright {\fontsize{10pt}{12pt}\selectfont RF}\huxbpad{3pt}}} &
\multicolumn{1}{p{0.135714285714286\textwidth}!{\huxvb{0, 0, 0}{0}}}{\hspace{6pt}\parbox[b]{0.135714285714286\textwidth-6pt-6pt}{\huxtpad{3pt + 1em}\raggedright {\fontsize{10pt}{12pt}\selectfont 5.63(0.4)}\huxbpad{3pt}}} &
\multicolumn{1}{p{0.135714285714286\textwidth}!{\huxvb{0, 0, 0}{0}}}{\hspace{6pt}\parbox[b]{0.135714285714286\textwidth-6pt-6pt}{\huxtpad{3pt + 1em}\raggedright {\fontsize{10pt}{12pt}\selectfont 4.61(0.3)}\huxbpad{3pt}}} &
\multicolumn{1}{p{0.135714285714286\textwidth}!{\huxvb{0, 0, 0}{0}}}{\hspace{6pt}\parbox[b]{0.135714285714286\textwidth-6pt-6pt}{\huxtpad{3pt + 1em}\raggedright {\fontsize{10pt}{12pt}\selectfont 4.27(0.4)}\huxbpad{3pt}}} &
\multicolumn{1}{p{0.135714285714286\textwidth}!{\huxvb{0, 0, 0}{0}}}{\hspace{6pt}\parbox[b]{0.135714285714286\textwidth-6pt-6pt}{\huxtpad{3pt + 1em}\raggedright {\fontsize{10pt}{12pt}\selectfont 3.21(0.2)}\huxbpad{3pt}}} &
\multicolumn{1}{p{0.135714285714286\textwidth}!{\huxvb{0, 0, 0}{0}}}{\hspace{6pt}\parbox[b]{0.135714285714286\textwidth-6pt-6pt}{\huxtpad{3pt + 1em}\raggedright {\fontsize{10pt}{12pt}\selectfont 3.57(0.4)}\huxbpad{3pt}}} &
\multicolumn{1}{p{0.135714285714286\textwidth}!{\huxvb{0, 0, 0}{0}}}{\hspace{6pt}\parbox[b]{0.135714285714286\textwidth-6pt-6pt}{\huxtpad{3pt + 1em}\raggedright {\fontsize{10pt}{12pt}\selectfont 2.62(0.2)}\huxbpad{3pt}}} \tabularnewline[-0.5pt]

\hhline{}
\arrayrulecolor{black}

\multicolumn{1}{!{\huxvb{0, 0, 0}{0}}p{0.135714285714286\textwidth}!{\huxvb{0, 0, 0}{0}}}{\hspace{6pt}\parbox[b]{0.135714285714286\textwidth-6pt-6pt}{\huxtpad{3pt + 1em}\raggedright {\fontsize{10pt}{12pt}\selectfont CF}\huxbpad{3pt}}} &
\multicolumn{1}{p{0.135714285714286\textwidth}!{\huxvb{0, 0, 0}{0}}}{\hspace{6pt}\parbox[b]{0.135714285714286\textwidth-6pt-6pt}{\huxtpad{3pt + 1em}\raggedright {\fontsize{10pt}{12pt}\selectfont 14.09(0.4)}\huxbpad{3pt}}} &
\multicolumn{1}{p{0.135714285714286\textwidth}!{\huxvb{0, 0, 0}{0}}}{\hspace{6pt}\parbox[b]{0.135714285714286\textwidth-6pt-6pt}{\huxtpad{3pt + 1em}\raggedright {\fontsize{10pt}{12pt}\selectfont 12.58(0.4)}\huxbpad{3pt}}} &
\multicolumn{1}{p{0.135714285714286\textwidth}!{\huxvb{0, 0, 0}{0}}}{\hspace{6pt}\parbox[b]{0.135714285714286\textwidth-6pt-6pt}{\huxtpad{3pt + 1em}\raggedright {\fontsize{10pt}{12pt}\selectfont 5.15(0.4)}\huxbpad{3pt}}} &
\multicolumn{1}{p{0.135714285714286\textwidth}!{\huxvb{0, 0, 0}{0}}}{\hspace{6pt}\parbox[b]{0.135714285714286\textwidth-6pt-6pt}{\huxtpad{3pt + 1em}\raggedright {\fontsize{10pt}{12pt}\selectfont 3.96(0.2)}\huxbpad{3pt}}} &
\multicolumn{1}{p{0.135714285714286\textwidth}!{\huxvb{0, 0, 0}{0}}}{\hspace{6pt}\parbox[b]{0.135714285714286\textwidth-6pt-6pt}{\huxtpad{3pt + 1em}\raggedright {\fontsize{10pt}{12pt}\selectfont 4.34(0.3)}\huxbpad{3pt}}} &
\multicolumn{1}{p{0.135714285714286\textwidth}!{\huxvb{0, 0, 0}{0}}}{\hspace{6pt}\parbox[b]{0.135714285714286\textwidth-6pt-6pt}{\huxtpad{3pt + 1em}\raggedright {\fontsize{10pt}{12pt}\selectfont 3.37(0.2)}\huxbpad{3pt}}} \tabularnewline[-0.5pt]

\hhline{}
\arrayrulecolor{black}

\multicolumn{1}{!{\huxvb{0, 0, 0}{0}}p{0.135714285714286\textwidth}!{\huxvb{0, 0, 0}{0}}}{\hspace{6pt}\parbox[b]{0.135714285714286\textwidth-6pt-6pt}{\huxtpad{3pt + 1em}\raggedright {\fontsize{10pt}{12pt}\selectfont Kennedy}\huxbpad{3pt}}} &
\multicolumn{1}{p{0.135714285714286\textwidth}!{\huxvb{0, 0, 0}{0}}}{\hspace{6pt}\parbox[b]{0.135714285714286\textwidth-6pt-6pt}{\huxtpad{3pt + 1em}\raggedright {\fontsize{10pt}{12pt}\selectfont 4.37(0.5)}\huxbpad{3pt}}} &
\multicolumn{1}{p{0.135714285714286\textwidth}!{\huxvb{0, 0, 0}{0}}}{\hspace{6pt}\parbox[b]{0.135714285714286\textwidth-6pt-6pt}{\huxtpad{3pt + 1em}\raggedright {\fontsize{10pt}{12pt}\selectfont 3.36(0.5)}\huxbpad{3pt}}} &
\multicolumn{1}{p{0.135714285714286\textwidth}!{\huxvb{0, 0, 0}{0}}}{\hspace{6pt}\parbox[b]{0.135714285714286\textwidth-6pt-6pt}{\huxtpad{3pt + 1em}\raggedright {\fontsize{10pt}{12pt}\selectfont 4.14(0.5)}\huxbpad{3pt}}} &
\multicolumn{1}{p{0.135714285714286\textwidth}!{\huxvb{0, 0, 0}{0}}}{\hspace{6pt}\parbox[b]{0.135714285714286\textwidth-6pt-6pt}{\huxtpad{3pt + 1em}\raggedright {\fontsize{10pt}{12pt}\selectfont 2.78(0.3)}\huxbpad{3pt}}} &
\multicolumn{1}{p{0.135714285714286\textwidth}!{\huxvb{0, 0, 0}{0}}}{\hspace{6pt}\parbox[b]{0.135714285714286\textwidth-6pt-6pt}{\huxtpad{3pt + 1em}\raggedright {\fontsize{10pt}{12pt}\selectfont 3.86(0.4)}\huxbpad{3pt}}} &
\multicolumn{1}{p{0.135714285714286\textwidth}!{\huxvb{0, 0, 0}{0}}}{\hspace{6pt}\parbox[b]{0.135714285714286\textwidth-6pt-6pt}{\huxtpad{3pt + 1em}\raggedright {\fontsize{10pt}{12pt}\selectfont 2.54(0.2)}\huxbpad{3pt}}} \tabularnewline[-0.5pt]

\hhline{}
\arrayrulecolor{black}

\multicolumn{1}{!{\huxvb{0, 0, 0}{0}}p{0.135714285714286\textwidth}!{\huxvb{0, 0, 0}{0}}}{\hspace{6pt}\parbox[b]{0.135714285714286\textwidth-6pt-6pt}{\huxtpad{3pt + 1em}\raggedright \textbf{{\fontsize{10pt}{12pt}\selectfont GCF}}\huxbpad{3pt}}} &
\multicolumn{1}{p{0.135714285714286\textwidth}!{\huxvb{0, 0, 0}{0}}}{\hspace{6pt}\parbox[b]{0.135714285714286\textwidth-6pt-6pt}{\huxtpad{3pt + 1em}\raggedright \textbf{{\fontsize{10pt}{12pt}\selectfont 4.14(0.3)}}\huxbpad{3pt}}} &
\multicolumn{1}{p{0.135714285714286\textwidth}!{\huxvb{0, 0, 0}{0}}}{\hspace{6pt}\parbox[b]{0.135714285714286\textwidth-6pt-6pt}{\huxtpad{3pt + 1em}\raggedright \textbf{{\fontsize{10pt}{12pt}\selectfont 2.88(0.2)}}\huxbpad{3pt}}} &
\multicolumn{1}{p{0.135714285714286\textwidth}!{\huxvb{0, 0, 0}{0}}}{\hspace{6pt}\parbox[b]{0.135714285714286\textwidth-6pt-6pt}{\huxtpad{3pt + 1em}\raggedright \textbf{{\fontsize{10pt}{12pt}\selectfont 4.05(0.4)}}\huxbpad{3pt}}} &
\multicolumn{1}{p{0.135714285714286\textwidth}!{\huxvb{0, 0, 0}{0}}}{\hspace{6pt}\parbox[b]{0.135714285714286\textwidth-6pt-6pt}{\huxtpad{3pt + 1em}\raggedright \textbf{{\fontsize{10pt}{12pt}\selectfont 2.7(0.3)}}\huxbpad{3pt}}} &
\multicolumn{1}{p{0.135714285714286\textwidth}!{\huxvb{0, 0, 0}{0}}}{\hspace{6pt}\parbox[b]{0.135714285714286\textwidth-6pt-6pt}{\huxtpad{3pt + 1em}\raggedright \textbf{{\fontsize{10pt}{12pt}\selectfont 3.85(0.4)}}\huxbpad{3pt}}} &
\multicolumn{1}{p{0.135714285714286\textwidth}!{\huxvb{0, 0, 0}{0}}}{\hspace{6pt}\parbox[b]{0.135714285714286\textwidth-6pt-6pt}{\huxtpad{3pt + 1em}\raggedright \textbf{{\fontsize{10pt}{12pt}\selectfont 2.48(0.2)}}\huxbpad{3pt}}} \tabularnewline[-0.5pt]

\hhline{>{\huxb{0, 0, 0}{0.4}}->{\huxb{0, 0, 0}{0.4}}->{\huxb{0, 0, 0}{0.4}}->{\huxb{0, 0, 0}{0.4}}->{\huxb{0, 0, 0}{0.4}}->{\huxb{0, 0, 0}{0.4}}->{\huxb{0, 0, 0}{0.4}}-}
\arrayrulecolor{black}
\end{tabularx}
\end{threeparttable}\par\end{centerbox}

\end{table*}

\subsection{Simulation}

Let $n$ be the number of samples and $p = p_X + p_U + p_Z$ be the dimension of covariates. The covariate matrix is $(X_i^j)_{i=1,\ldots,n}^{j = 1,\ldots,p_X} = (\boldsymbol{X}_1,\ldots,\boldsymbol{X}_n) \in \mathbb{R}^{n \times p_X}$. DGP is formally written as
\begin{align*}
    & Y = \mu(T) + 0.2(X^1 \cdot X^1 + X^4)T + \boldsymbol{X} \cdot \boldsymbol{\beta_X} +\boldsymbol{U} \cdot \boldsymbol{\beta_U} + \epsilon \\
    & T = 20 \cdot \Psi(\phi(\boldsymbol{X} \cdot \boldsymbol{\beta_X} + \boldsymbol{Z} \cdot \boldsymbol{\beta_Z})) + \nu
\end{align*}
where $\phi$ is the sigmoid function and $\Psi$ is the pdf of Beta distribution. Here we set DRF $\mu$ to be polynomial(Poly), exponential(Exp), and sinusoidal functions(Sinus).

We now have 6 datasets by considering every possible combination of 3 DRFs $\mu(t)$ and 2 setups of covariate matrix $\Sigma$. Noises $\epsilon$, $\nu$ follow $Unif(-1,1)$. The Covariates $\boldsymbol{X}$ and coefficients $\beta$ follow $N(0,I_{p_X})$. We also allow sparsity in the covariates by randomly setting some coefficients to be 0. We elaborate the details and the choices of hyperparameters in the supplementary.

\begin{figure}[h] 
  \centering
  \includegraphics[width=0.85\linewidth]{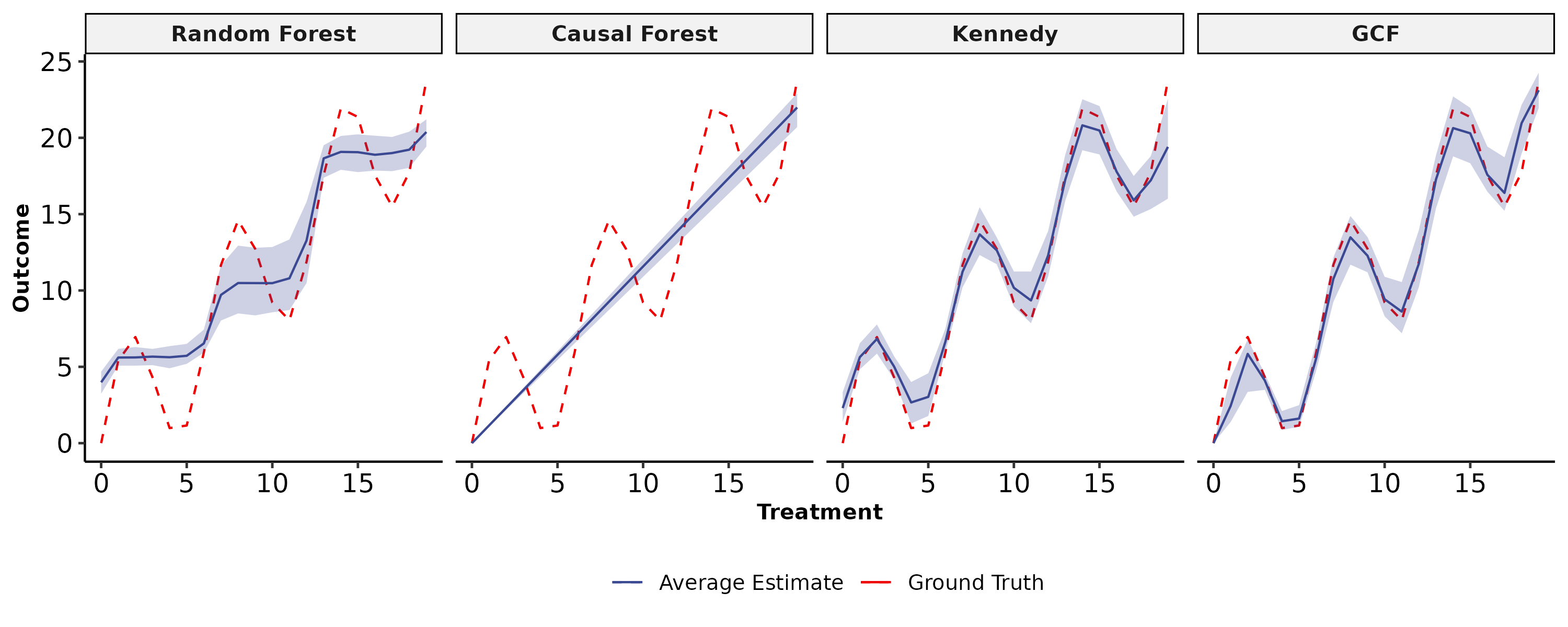}
  \caption{Comparison of ADRF estimation for different models. The ground truth is in red, the blue line denotes the median and the light-blue confidence bands are the 2.5$\%$ and 97.5$\%$ quantiles of these estimates across 100 simulations.}\label{fig:adrf_2}
\end{figure}



The simulation results on 3 DGPs are summarized in Table \ref{tab: RMSE}. Here PEHE and RMSE are averaged over 100 simulation runs and their standard error are also reported. Overall, GCF consistently outperforms baseline methods in that GCF exhibits the smallest biases and variances for multiple DGP setups. 

Additionally, we compare the Average DRF (ADRF) curve given by different models. An example of ADRF for one dataset is shown in Fig. \ref{fig:adrf_2}. Among those models, the curve of GCF is closest to the ground truth.

\subsection{Real-world Datasets}
Our method is tested on a real-world dataset with 10,698,884 entries collected from a randomized experiment conducted on a ride-sharing company where the treatment is a discount. Discounts assigned to ODTs are randomly sampled from a set of options $\{d_0,d_1,d_2$ $,d_3,d_4,d_5: d_0 < d_1 < \ldots < d_5\}$. The effect of discount on the demand is the estimand of interest. We compare GCF with CF and Xgboost \citep{chen2015xgboost} and evaluate the performance using Qini score. 

\begin{table}[H] 
  \centering
  \caption{Qini scores of models under different treatments}
  \label{tab:Qini1}
  \begin{tabular}{clllll}
    \hline
    Methods & $d_5$ & $d_4$ & $d_3$ & $d_2$ & $d_1$\\
    \hline
    Xgboost & 0.253 & 0.171 & 0.177 & 0.206 & 0.177 \\
    CF & 0.253 & 0.194 & 0.202 & 0.272 & 0.300 \\
    \textbf{GCF} & \textbf{0.309} & \textbf{0.248} & \textbf{0.305} & \textbf{0.444} & \textbf{0.780} \\
    \hline
\end{tabular}
\end{table}

The Qini score of different models are summarized in Table~\ref{tab:Qini1}. The performance of GCF is superior than the others by noting that GCF has the highest Qini score across all options of discounts.

\section{Implementation and Deployment}
We deploy our algorithm to the online pricing system of a leading ride-sharing company. The system is designed to deliver optimal pricing strategies which supports more than 550 million riders and tens of millions of drivers world-wide everyday. Given such a huge amount of data, we implement GCF on Spark to speed up model training by means of distributed computations. As illustrated in Figure~\ref{fig:onlineflow}, the system starts with collecting real-world data from the experimental system. In what follows, data is sent to the model training module where GCF and other baseline models are trained. Subsequently, the best model selected by tailored evaluation metrics (e.g., Qini score) provides treatment effect predictions for the policy optimization module, which generates a global-optimal pricing strategy for the online service.

\begin{figure}[h] 
  \centering
  \includegraphics[width=.4\textwidth]{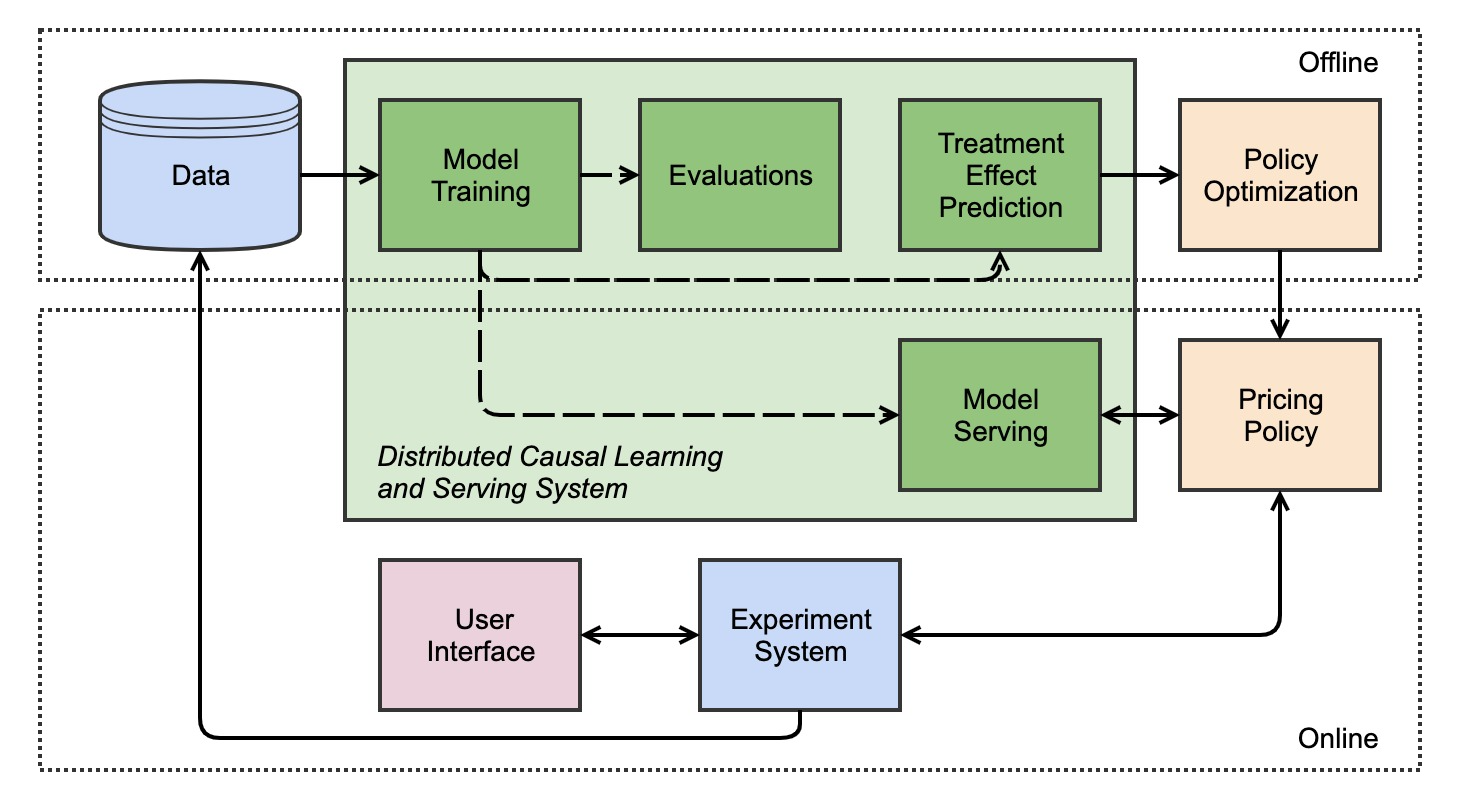}
  \caption{A schematic diagram of an online pricing system}\label{fig:onlineflow}
\end{figure}  

To examine the empirical effectiveness of our model, we compare the discount strategies resulted from GCF and CF under two business settings using online A/B testing. We conduct online A/B testing by randomly splitting ODTs into two groups. Note that data considered here takes only a small portion of the whole market, which implies that network effect can be disregarded. The key metric for performance evaluation is the increment on finished orders (FO), of which the results are as follows. Compared to CF, GCF improves FO by 15.1$\%$ and 25.2$\%$ in the single mobility options strategy and dual mobility options strategy, respectively. It shows that our model can better estimate the treatment effects on complicated systems.

\section{Conclusion}
\label{sec:conclusion}
In this paper, we propose a novel forest-based non-parametric algorithm, namely generalized causal forest, to address the problem of HTE estimation with continuous treatments. We extend CF by introducing DRF with a generic distance-based splitting criterion that maximize heterogeneity of continuous treatment effects. To estimate the DRF, we use the kernel-based doubly robust estimator to guarantee double robustness. To handle huge amount of data, we implement GCF on Spark and successfully deploy it on the online pricing system in a leading ride-sharing company. The empirical results demonstrate that our method significantly outperforms competing methods. 

Within the scope of this paper, we only cover the case of one-dimensional continuous treatment. But what we propose can be extended to multi-dimensional case without additional efforts. It also worth mention that, kernel regression may substantially suffer from the curse of dimensionality when the treatment space is high and sparse. More robust algorithms on HTE estimations for high-dimensional treatments are promising as a future research area.


\bibliographystyle{ACM-Reference-Format}
\bibliography{arxiv}

\clearpage

\clearpage

\newpage
\section{Supplementary Materials} 

\subsection{Notations}

We restate the notations that are consistent with the main paper. Following the potential outcome framework in~\citep{neyman1923applications, rubin1974estimating}, we let $T$ be the continuous treatments, $\boldsymbol{X} = (\boldsymbol{X}^j)_{j=1}^{p}$ be the $p_{X}$-dim confounder, $\boldsymbol{U}$ be the $p_{U}$-dim outcome-specific adjustment variable, $\boldsymbol{Z}$ be the $p_{Z}$-dim treatment-specific adjustment variable, and $Y$ be the observed outcome.The potential outcomes under treatment $t$ is $Y_{(t)}$. The population $\Omega = (\boldsymbol{\Sigma} = (\boldsymbol{X}, \boldsymbol{U}, \boldsymbol{Z}), Y, T) \in \mathbb{R}^{p_{X}+p_{U}+p_{Z}+1}$ satisfies 
\begin{align*}
   & Y = g(T,\boldsymbol{X},\boldsymbol{U}) + \epsilon , T = f(\boldsymbol{X},\boldsymbol{Z}) + \nu \text{ where } \epsilon,\nu \text{ are standard noises,} \\
   & \text{and } g: \mathbb{R}^{p_{X}+p_{U}} \times \mathbb{R} \rightarrow \mathbb{R} \text{ and }  f: \mathbb{R}^{p_{X}+p_{Z}}  \rightarrow \mathbb{R} .
\end{align*}
$\{(\boldsymbol{\Sigma}_i, Y_i, T_{i}), i=1, \ldots, n\}$ are i.i.d. samples drawn from the population $\Omega$. Then the covariate matrix $ \Sigma = (\Sigma_i^j)_{1 \leq i \leq n}^{1 \leq j \leq p} = (\boldsymbol{\Sigma}_1,\ldots,\boldsymbol{\Sigma}_n) \in \mathbb{R}^{n \times p}$ where $p= p_{X}+p_{U}+p_{Z}$.  The generalized propensity score is $\pi(\boldsymbol{T} = \boldsymbol{t}|\boldsymbol{X})$, which is the probability density for a unit receiving treatment $t$ given the covariate $\boldsymbol{X}$. 

\subsection{Proof of Theorem 1} \label{ssec:proof1}

We start with assumptions. 

\begin{assumption}
The pdf of the joint distribution $f(x,y,t)$ is three-times differentiable.
\end{assumption}

\begin{assumption}
The second-order symmetric kernel $k(\cdot)$ is bounded diffentiable. 
\end{assumption}

\begin{assumption}
Suppose that the estimators $\hat{\pi}$ and $\hat{\mu}$ satisfies 
\begin{align*}
& \sigma_x(\hat{\mu}) = \int_x(\hat{\mu}(t,x)-\mu(t,x))^2 \cdot f(t,x)dx \overset{p}{\to} 0 \\
& \sigma_x(\hat{\pi}) = \int_x(\hat{\pi}(t|x)-\pi(t|x))^2 \cdot f(t,x)dx \overset{p}{\to} 0 \\
& \sqrt{nh} \sqrt{\sigma_x(\hat{\mu})} \cdot \sqrt{\sigma_x(\hat{\pi})} \overset{p}{\to} 0 
\end{align*}
\end{assumption}

\begin{assumption}
Suppose that $\psi_{\theta_x,\nu(x)}$ is Lipschitz in $\theta_x$ and $\theta_x$ is lipschitz in $x$ and $\norm{\theta_x} \leq C $ for some constants $C$.
\end{assumption}

\begin{assumption}
Suppose that $\psi_{\theta_x,\nu(x)}$ is twice continuously differentiable in $\theta(x)$ and $\theta(x)$ is convex in $t$. 
\end{assumption}

\begin{assumption}
   Suppose that $\psi_{\theta_x,\nu(x)}(\omega_i) = \theta_{x}(T_i) - Y_i + \nu(x)$ is a negative gradient of a convex function and the expected score function $E[\psi_{\theta_x,\nu(x)}(\omega_i)]$ is a negative gradient of a strongly convex function. 
\end{assumption}

We first show that the doubly robust estimator for DRF and ATE is asymptotically unbiased which is crucial for the theoretical guarantee of the successive estimators for CDRF and CATE.

\begin{proof}
Under assumptions 1-3 and 5-7, we have Lemma~\ref{the:lemm} hold. 
\begin{lemma}\label{the:lemm}
Let Assumptions 1-3 and 5-7 hold.  The doubly robust estimator $\tilde{\mu}$ satisfies that for any given $t$,
\begin{align*}
    \sqrt{nh^{d_t}}(\tilde{\mu}(\boldsymbol{t}) - \mu(\boldsymbol{t})) &\to N(0,V^1_t), \quad
    \tilde{\mu}(t) - \mu(t) &\overset{p}{\to} 0
\end{align*}
where $V^1_t = E\left[\dfrac{var(Y|\boldsymbol{T} = \boldsymbol{t},X=x)}{\pi(t|x)}\right]\cdot \int_{-\infty}^{\infty}k(u)du < \infty$ as in ~\cite{colangelo2020double}. 
\end{lemma}

\begin{proof}
It is straightforward to check that all the assumptions in~\cite{colangelo2020double} hold when utilizing the optimal bandwidth $h$ that satisfies $h^2\sqrt{nh^{d_t}} \to 0$ Then we close the proof by directly following the result in~\cite{colangelo2020double}. 
\end{proof}

With the asymptotically unbiased estimator plugged in to the splitting criterion, now we adapt what are established in~\cite{athey2019generalized} to our GCF. We first define the score function in general scenarios with non-linear DRF, which is denoted by $\psi_{\theta_x,\nu(x)}(\omega_i) = \theta_{x}(\boldsymbol{T_i}) - Y_i + \nu(x)$ on sample $\omega_i = (\boldsymbol{T_i}, Y_i)$ where $\theta_{x}(\cdot) = \mu(\boldsymbol{t},x) - \mu(\boldsymbol{0},x) =\mu(\boldsymbol{t},x) - v(x)$ is the mapping from $\mathcal{R}^{d_t}$ to $\mathcal{R}$ or the distance between two points on the dose response curve. Note that the newly defined score function generalizes the parametric $\theta(x)$ as in $\theta(x)\cdot T$ to the generic one $\theta_{x}(\cdot)$, i.e. from $ \mathcal{R}$ to $ \{h:\mathcal{R}^{d_t} \to \mathcal{R}\}\}$. 

Solving an estimating equation of score function $\Psi$ is at the heart of the analysis in~\cite{athey2019generalized}. While it does not rely on the linear assumption, we can similarly propose an estimating equation as 
\begin{align*}
   \norm{\Psi_{\theta,\nu}} = \norm{E[\psi_{\theta,\nu}(\omega)]} = 0
\end{align*}
And the empirical estimating equation gives it to 
\begin{align*}
    & min_{\theta,\nu} \norm{\sum_{i=1}^{n} \alpha_i \cdot \psi_{\theta,\nu}(\omega_i)}  = \norm{\dfrac{\sum_i \theta_x(\boldsymbol{t},\boldsymbol{x_i}) + \nu(\boldsymbol{x_i})}{s} - \dfrac{\sum_i Y_i}{s}} \\
    & = \norm{\dfrac{\sum_i \mu(\boldsymbol{t},\boldsymbol{x_i}) - \mu(\boldsymbol{0},\boldsymbol{x_i}) + \mu(\boldsymbol{0},\boldsymbol{x_i})}{s}  - \dfrac{\sum_i Y_i}{s}}
\end{align*}
where $\alpha_i$ is the weight as defined in (\ref{equa:1}) and $s$ is the number of samples with non-zero weights. Here we propose that with $\tilde{\mu}(\boldsymbol{t})$, we can get an asymptotically optimal solution to the empirical equation as in~\cite{athey2019generalized} since we have that 
\begin{align*}
&\norm{\dfrac{\sum_i \mu(\boldsymbol{t},\boldsymbol{x_i}) - \mu(\boldsymbol{0},\boldsymbol{x_i}) + \mu(\boldsymbol{0},\boldsymbol{x_i})}{s}  - \dfrac{\sum_i Y_i}{s}} \\
 & \overset{p}{\to} \norm{\mu(\boldsymbol{t}) - \mu(\boldsymbol{0})+\mu(\boldsymbol{0})  - \dfrac{\sum_i Y_i}{s}} \thickapprox  \norm{\tilde{\mu}(\boldsymbol{t}) - \tilde{\mu}(\boldsymbol{0})+\tilde{\mu}(\boldsymbol{0})  - \dfrac{\sum_i Y_i}{s}} \\ 
\end{align*}
where the first convergence is by SLLN and the second approximation holds by Lemma~\ref{the:lemm} aforementioned when $s \geq  \sqrt{nh^{d_t}} \to \infty$.

To close the proof, we show that by solving the empirical estimating equation that is utilized in the splitting criterion ~\ref{splitting creterion}, the final estimator for CATE given by the tree is asymptotically unbiased. This is done by generalizing the theorem 1 in~\cite{athey2019generalized}.

\begin{remark}{Regularity:}
\begin{align*}
\psi_{\theta_x,\nu(x)}(\omega_i) & = \theta_{x}(T_i) - Y_i + \nu(x) = \lambda(\theta_{x},\nu(x);x) + \eta(g(\omega_i)) \end{align*}
where $\lambda(\theta_{x},\nu(x);x) = \theta_{x}(T_i)  + \nu(x)$ and $\eta = -I$ and $g(\omega_i) = Y_i$ satisfy the Regularity as Assumption 4 as in~\cite{athey2019generalized}.
\end{remark}

\begin{remark}{Existence of Solutions:}
The estimators for the nuisance functions are given by~\cite{colangelo2020double}. Then under Assumption 1-3 and 5-7, the convergence result as in~\cite{colangelo2020double} hold as implied by Lemma~\ref{the:lemm}. That is to say that the estimators $\hat{\theta},\hat{\nu}$ for $\theta, \nu$ given by the subsamples on the child node lead to a solution to the optimization problem 
\begin{align*}
    min_{\theta,\nu} \norm{\sum_{i=1}^{n} \alpha_i \cdot \psi_{\theta,\nu}(\omega_i)}
\end{align*}
Since the score function itself is Lipschitz which implies that it is continuous, Weierstrass theorem gives us that the above optimization problem exists a global optimum which implies the $\sum \alpha_i \cdot M = 0$. Therefore, Assumption 5 as in~\cite{athey2019generalized} hold in our scenario. This is the key step for guaranteeing the effectiveness of doubly robust estimator in the framework of CF. 
\end{remark}

\begin{remark}{Specification 1:}
From Lemma~\ref{the:lemm}, we could get that with $\sqrt{nh^{d_t}}$ samples, the estimator for DRF converges to the ground truth. Then the number of samples $s$ can be specified as $\sqrt{nh^{d_t}}$ and it satisfies  $\dfrac{s}{n} \to 0$ and $s \to \infty$. It follows Specification 1 as in~\cite{athey2019generalized}
\end{remark}

According to Remark 2 and Remark 3, the conditions of existence of solutions as denoted by Assumption 5 and Specification 1 in~\cite{athey2019generalized} is met. Meanwhile, the effectiveness of Assumption $8$ implies that the score function is Lipschitz and the expect score function is Lipschitz which are stated as Assumption 1 and 3 in~\cite{athey2019generalized}. It is easy to check that Assumption 2 in~\cite{athey2019generalized} holds when Assumption 9 holds herein. Remark 1 justifies the validity of Assumption 4 in~\cite{athey2019generalized}. Lastly, Assumption 6~\cite{athey2019generalized} is guaranteed by Assumption 10 aforementioned.  Since Assumption 1-6 and Specification 1 in~\cite{athey2019generalized} are met, we go through the proof steps of theorem 1~\cite{athey2019generalized} line by line with functional norm and distance and get
$\norm{\hat{\theta}_x(t) - \theta_x(t)}_F \to 0 \text{ as }  n \to \infty$.\end{proof}

\subsection{Data Generating Process}

Recall that the covariate matrix $ \boldsymbol{\Sigma} \in \mathbb{R}^{n \times (pX + pU +pZ)} = (X_i^j, U_i^j, Z_i^j) $ and treatment $T \in \mathbb{R}$ where $n$ is the number of observations and $p$ is the dimension of covariates. Data generating process (DGP) is
\begin{align*}
    & Y = \mu(T) +   0.2(\boldsymbol{X_1}^2 + \boldsymbol{X_4})T + \boldsymbol{X} \cdot \boldsymbol{\beta_X} +\boldsymbol{U} \cdot \boldsymbol{\beta_U} + \epsilon \\
    & T = 20 \cdot \Psi(\phi(\boldsymbol{X} \cdot \boldsymbol{\beta_X^*} + \boldsymbol{Z} \cdot \boldsymbol{\beta_Z})) + \nu
\end{align*}
where $\phi$ is the sigmoid function and $\Psi$ is the pdf of Beta distribution with shape parameters set to 2 and 3. Here we set DRF $\mu$ as polinomial(Poly), exponential(Exp), and sinusoidal functions(Sinus).
\begin{align*}
    & \mu(t) = 0.2 \cdot (t - 5)^2 - t - 5 \text{, Polynomial}\\
    & \mu(t) = log\left(1 + \dfrac{exp(t)} {t + 0.1}\right) - log(11) \text{, Exponential} \\
    & \mu(t) = 5 \cdot sin(t) + t \space \text{, Sinusoidal}
\end{align*}
The above DGP gives us multiple datasets by taking the combination of DRF $\mu(t)$ and covariate $\Sigma$ that are specified by $n$,$p$ and $DRF$. Parameter $\epsilon$, $\nu$ are Gaussian noises. Covariates $\boldsymbol{X, U, Z}$ follows $N(0,\boldsymbol{I})$ and coefficient vector $\boldsymbol{\beta}$ follow $\left[Unif(-1,1)\right]$. Following this rule, we generate 100 rounds. Meanwhile, in test data, we randomly assigned the treatments to make sure unbiased evaluations.

\subsection{Hyper-parameters}

Both RF and CF use $num.trees$ equals to 500, $min.node.size$ to 50. For Kennedy \citep{kennedy2017nonparametric}, we use the code of function $cts.eff$ from R package $npcausal$ \citep{codeKennedy2017}. SuperLearner library is set to $SL.ranger$ and $SL.glm$. We made small tweaks to $SL.ranger$ to make share same hyperparameters with RF and CF for comparability.

\subsection{Practical Considerations}

The choice of bandwidth $h$ for kernels weighs more than the choice of density functions~\citep{kennedy2017nonparametric}, since $h$ balances between the bias and the variance. Usually, a small $h$ avoids a large variance while a large $h$ reduces the bias. The empirical ways of choosing an optimal $h$ are Cross Validation and Rule-of-thumb~\citep{dehnad1987density}. 

When utilizing kernel densities for weighting, the algorithm often suffer from the boundary bias if without normalization. To this end, we normalize estimators with the cumulative density function of range $[T_{min},T_{max}]$ of treatments. Formally, the estimators are divided by 
$\int_{T_{min}}^{T_{max}} K_h(T_i-t) dt$. 

Regarding the positivity assumption in practice, we use a hyperparameter $\zeta$ to control for a strictly positive variance of GPS $\sigma(\pi)$. This is also a guarantee on the weak assumption as stated in Assumption \ref{assumption:4}. Formally, our proposed splitting criterion is$\Delta(C_1,C_2) = \dfrac{n_{C_1}n_{C_2}}{n_P} \cdot D(\hat{\theta}_{C_1}, \hat{\theta}_{C_2}) + \zeta \cdot \sigma(\hat{\pi})
   \label{splitting creterion}$
where $\hat{\pi}$ is the estimated GPS. This splitting criterion considers the variance of GPS and encourages the variety of treatments. 
\subsection{Spark Implementation}

Apache Spark \citep{meng2016mllib} has the power of large-scale data processing and provides APIs of any machine learning algorithms. Consequently, we build our proposed model on Spark to do parallel computing and distributed model training that leverage the resources of multiple machines simultaneously. The workflow of Spark with GCF is shown in Figure \ref{fig:spark_app}, which include integrated data and distributed computation. More specifically, with the parallel structure of spark, the tree-growing process runs and the process of GCF is depicted in Figure \ref{fig:spark_diag}, which is significantly different from that in CF \citep{athey2019generalized} by distributing tasks for computations.

\vspace{-2mm}
\begin{figure}[h] 
  \centering
  \includegraphics[width=0.4\textwidth]{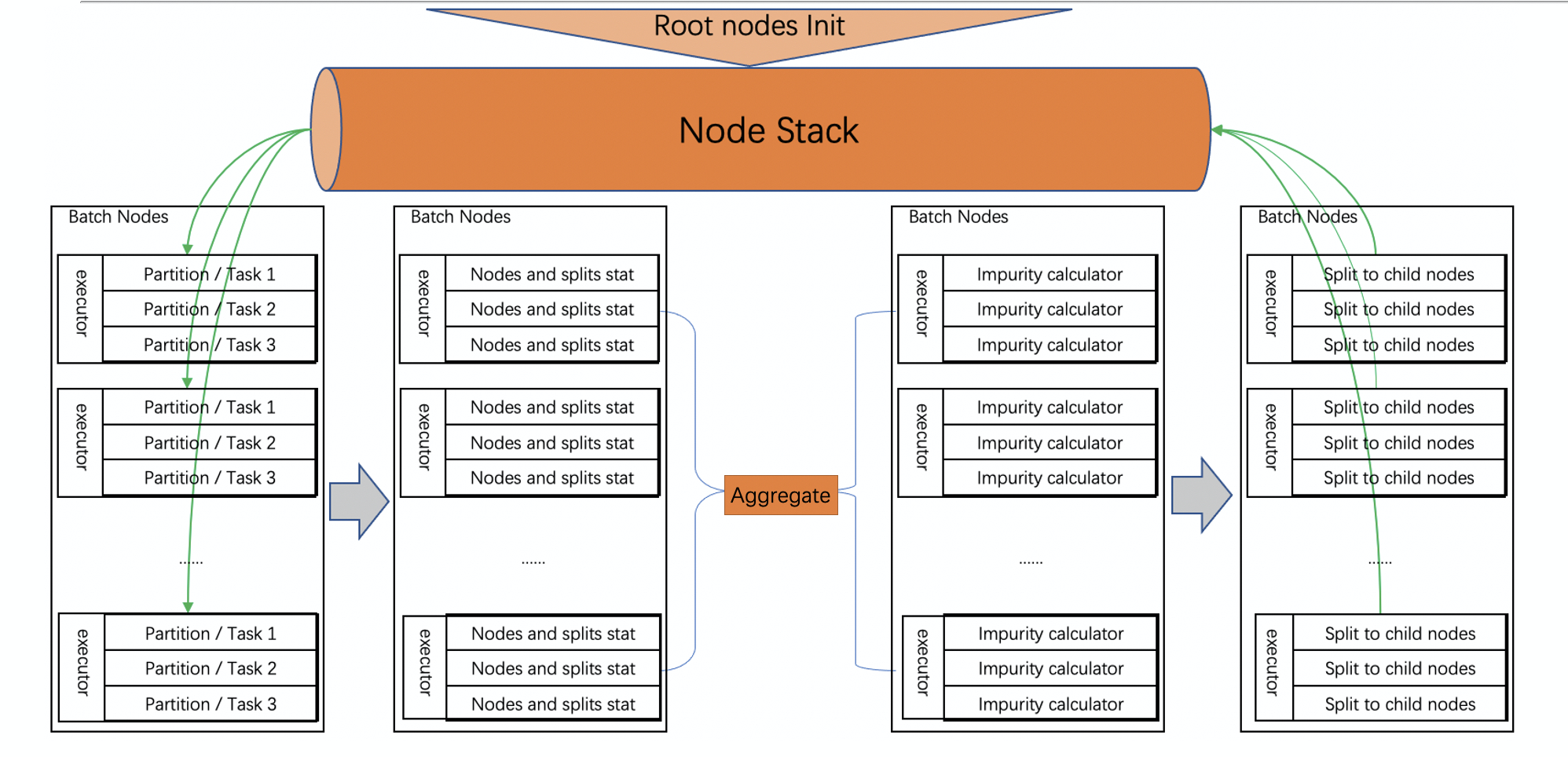}
  \caption{Parallel architecture of GCF Spark Implementation. Data are partitioned at the nodes level and distributed across executors to achieve high performance}\label{fig:spark_app}
\end{figure} 
\vspace{-6mm}

\begin{figure}[H] 
  \centering
  \includegraphics[width=0.3\textwidth]{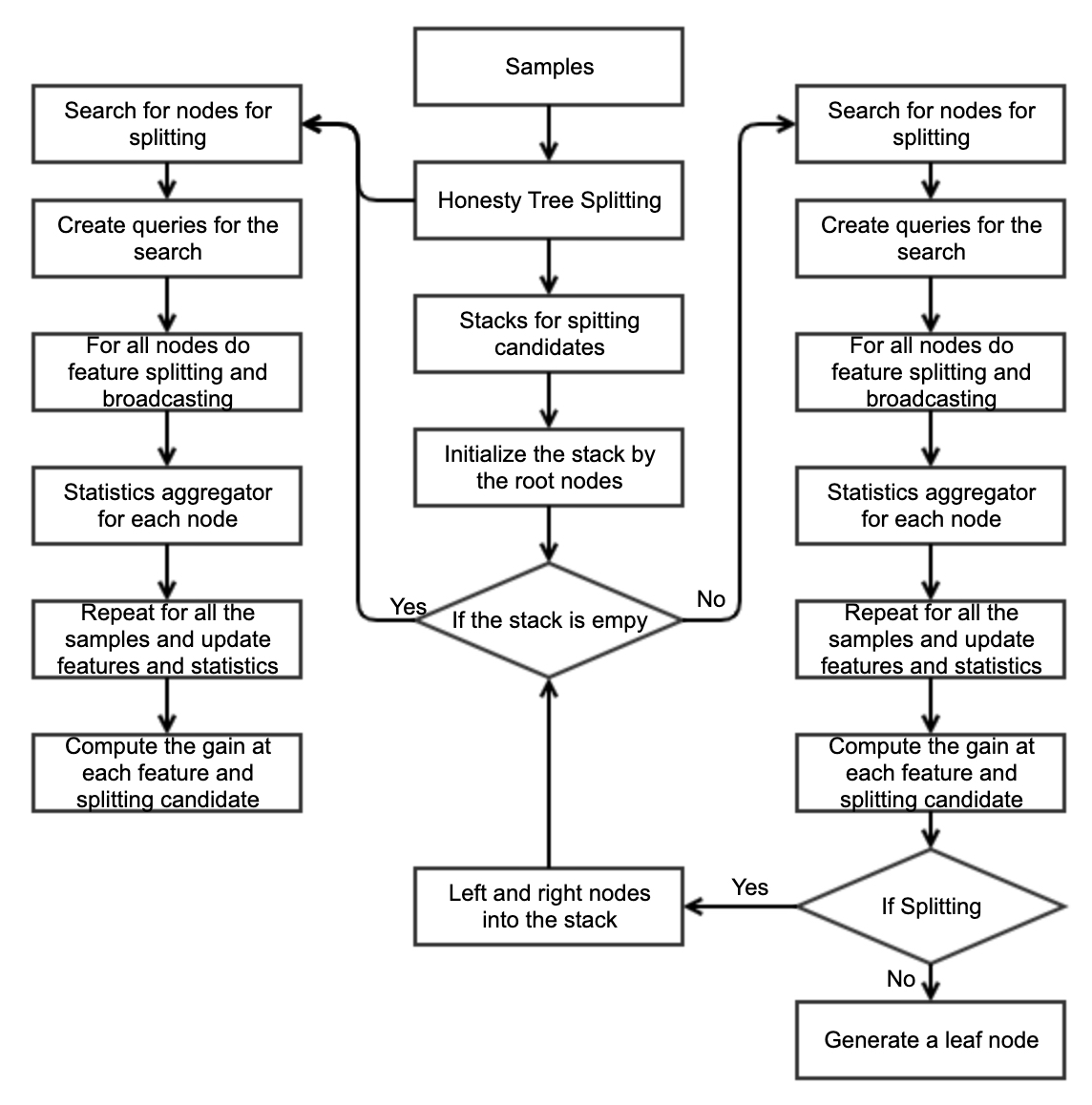}
  \caption{GCF algorithm workflow}\label{fig:spark_diag}
\end{figure}

\end{document}